\documentclass[acmsmall,screen,nonacm]{acmart}
\setcopyright{none}
\makeatletter
\@ifundefined{hyxmp@parse@acmart}{}{\let\hyxmp@parse@acmart\relax}
\makeatother
\usepackage{amsmath,mathtools,booktabs,tabularx,array,enumitem,microtype}
\usepackage{tikz}
\usetikzlibrary{arrows.meta,positioning,fit,calc}
\hypersetup{keeppdfinfo=true,hidelinks,pdftitle={Evidence-Guided Recovery for AI Agents with Ambiguous Tool Outcomes and Partial Execution},pdfauthor={Genliang Zhu; Xiaobo Chen; Ziyuan Wang; Chu Wang},pdfcreator={LaTeX}}
\setlist[itemize]{leftmargin=1.25em}
\setlist[enumerate]{leftmargin=1.5em}
\newcolumntype{Y}{>{\raggedright\arraybackslash}X}
\newtheorem{proposition}{Proposition}
\newcommand{\MechanismTemplates}{20}
\newcommand{\MechanismAssignments}{37}
\newcommand{\MechanismRecords}{185}

\newcommand{\JointSelectionSeconds}{9.53}
\newcommand{\RulesSelectionSeconds}{0.14}
\newcommand{\JointSelectionPninetyfive}{288.53}
\newcommand{\RulesSelectionPninetyfive}{3.55}

\newcommand{\DiagnosticTasks}{21}
\newcommand{\DiagnosticPairs}{63}
\newcommand{\DiagnosticRootValues}{408}
\newcommand{\DiagnosticBudgetPairs}{126}
\newcommand{\DiagnosticPartialLayers}{111}
\newcommand{\DiagnosticFallbacks}{27}
\newcommand{\OriginalExtraSeconds}{9.585}
\newcommand{\OriginalBreakEven}{2.396}

\newcommand{\RecPairs}{58}
\newcommand{\RecFamilies}{19}
\newcommand{\RecEpisodes}{116}
\newcommand{\RecEvidence}{9}
\newcommand{\RecPlanner}{30}

\newcommand{\RecDifference}{36.2}
\newcommand{\RecP}{0.0078}
\newcommand{\RecEvidencePct}{15.5}
\newcommand{\RecPlannerPct}{51.7}

\title{Evidence-Guided Recovery for AI Agents with Ambiguous Tool Outcomes and Partial Execution}
\author{Genliang Zhu}
\affiliation{\institution{Accentrust}\city{Vancouver}\country{Canada}}
\affiliation{\institution{Georgia Institute of Technology}\department{School of Computer Science}\city{Atlanta}\state{Georgia}\country{USA}}
\email{research@accentrust.com}
\author{Xiaobo Chen}
\affiliation{\institution{University of Massachusetts Amherst}\city{Amherst}\state{Massachusetts}\country{USA}}
\author{Ziyuan Wang}
\affiliation{\institution{Duke University}\city{Durham}\state{North Carolina}\country{USA}}
\author{Chu Wang}
\affiliation{\institution{Accentrust}\city{Vancouver}\country{Canada}}
\affiliation{\institution{University of Illinois Urbana-Champaign}\department{Department of Computer Science}\city{Urbana}\state{Illinois}\country{USA}}

\begin{document}
\begin{abstract}
Tool timeouts can leave AI agents facing ambiguous tool outcomes and partial execution: an action may have failed, committed or remained in flight while other workflow steps have already completed. We formulate evidence-guided task recovery over approved obligations, provider capabilities and states compatible with public evidence. A shared admissibility check handles key expiry, authorization timing, dependencies and residual effects; a terminal certificate separates actual goal satisfaction from publicly supported safe completion. The known-model implementation combines standard observation-contingent planning with a separately implemented effect-history adjudicator. In a finite study of 37 assignments, planning and three active controls each certify 28 completions; four saved synthetic fee units incur additional computation. A preregistered assignment-held-out comparison then evaluates planning advice against evidence prompting with two local model bundles. Across \RecPairs{} paired units in \RecFamilies{} previously developed task families, advice certifies \RecPlanner{} completions versus \RecEvidence{}, a \RecDifference{}-percentage-point difference (exact family sign-flip $p=\RecP$), using 272 proposals versus 399. All \RecEpisodes{} trajectories pass independent replay. Both models exhibit a positive completion difference. The comparison measures the evaluated advice interface under a common shield; search-versus-rule advice requires a separate matched comparison. Unknown and handoff outcomes remain explicit, with guarantees tied to the declared provider capabilities rather than arbitrary third-party exactly-once execution.
\end{abstract}
\maketitle
\section{Introduction}
An agent that creates an invoice, updates a business record, and notifies a customer interacts with several independently operated services. A timeout at the first service does not reveal whether the invoice was created. A batch update may have changed some records before returning an error. A notification can be delivered while its acknowledgment is lost. Repeating the entire workflow can duplicate irreversible effects; stopping can leave legitimate obligations unfulfilled. The operational problem is therefore not only to avoid another write, but to determine which remaining work can safely be completed.

The distinction between a failed invocation and an unknown remote outcome is fundamental to remote procedure calls~\cite{birrell1984rpc}. Durable identifiers, receiver-side deduplication, and application-level evidence can improve recovery, but their guarantees depend on the endpoints that participate. A local transaction cannot make an uncooperative remote service atomic. End-to-end arguments explain why correctness conditions that depend on an application's effects require support at the relevant endpoints~\cite{saltzer1984endtoend}. Compensation provides a different form of recovery: a permitted later action can repair some state, without erasing the fact that the original action occurred~\cite{garcia1987sagas}.

Agent research now addresses these problems directly. SagaLLM combines persistent context, validation, and compensating recovery~\cite{chang2025sagallm}. Atomix distinguishes transaction settlement from execution progress and coordinates different effect classes~\cite{mohammadi2026atomix}. LIMBO studies how models, harnesses, and tool contracts affect duplicate side effects, including late commits and limited read paths~\cite{li2026limbo}. UndoBench separates nominal task competence from fault recovery using paired executions and both effect-history and state evaluation~\cite{sah2026undobench}. Building on these recovery and evaluation foundations, we focus on selecting evidence acquisition and continued execution for the remaining task under a fixed provider contract.

We study the decision problem that remains when the provider's capabilities are fixed. A recovery controller must choose whether to spend resources obtaining evidence, wait for an observation to become decisive, retry under a valid deduplication contract, continue independent obligations, compensate a reversible step, or hand the task to an authorized operator. These choices interact. Waiting for one uncertain step can consume the time needed for another obligation. A cheap query can be inconclusive until a visibility delay has passed. Compensation can create its own ambiguous outcome. Retrying before a key expires can be safe under one contract and unsafe under another. An operator may not respond before the business deadline.

The controller receives the approved task, public capability contracts, public observations, and a supplied finite transition model. That model includes the candidate initial commit times, complete future write/compensation attempt sequences, and their weights. It does not receive the realized assignment or private effect ledger. Thus, this is known-model partial-observation planning, not recovery from a timeout using capability flags alone. We represent uncertainty by all states compatible with these inputs. A common admissibility check excludes actions that add a checked violation or abstract pending hazard in any compatible state. Within that common action set, a finite-horizon planner compares evidence acquisition and continued execution. Probabilities guide selection; they cannot turn a low-probability prohibited transition into an admissible action.

The contribution is a recovery decision interface and its controlled evaluation, built on established observation-contingent planning:
\begin{enumerate}
\item An explicit contract connects remaining obligations, public provider evidence and residual in-flight effects. It distinguishes a correct current snapshot from certified completion stable at every accepted finite-tail prefix.
\item An executable known-model recovery profile compares common action checks, capability rules, one-step selection and standard observation-contingent search under matched observations, authority and resources. Separately implemented adjudication and exhaustive development diagnostics check effects and completed search layers.
\item A source-bound empirical decomposition separates deterministic policy behavior from model-facing advice. A registered assignment-held-out study records 116 model episodes: planning advice certifies 30/58 paired units versus 9/58 with evidence prompting, while strong deterministic controls already tie at 28/37 in their separate frame.
\end{enumerate}
Table~\ref{tab:recovery-evidence} maps these claims to their evidence. The model-facing comparison measures the evaluated advice interface under its common shield; isolating search depth would require matched alternative advice. Weights describe the supplied finite model, and the resource analysis reports synthetic fees separately from selection computation and deployment costs.

The intended application is a cross-system coordinator that preserves useful progress and gives an actionable explanation when completion cannot be certified. The coordinator does not grant additional authority, manufacture provider capabilities, or promise exactly-once completion for arbitrary services. A task with no trustworthy way to distinguish a pending effect from a completed one can legitimately remain unresolved. Such an outcome is kept separate from task completion, successful compensation, and a merely plausible success message.

\begin{table}[t]
\caption{Recovery claims and executed evidence. The common shield, policy choice and model interface are distinct layers.}
\label{tab:recovery-evidence}
\centering\small
\begin{tabularx}{\linewidth}{p{0.20\linewidth}YY}
\toprule
Object & Evidence & Interpretation\\
\midrule
Evidence and completion & Separate effect-history replay and finite-tail checks & Public completion differs from actual private goal satisfaction; unknown remains explicit.\\
Policy selection & 37 assignments, five policies; complete-layer and root-vector audits & Four active policies tie at 28/37; search trades synthetic fee savings for computation.\\
Model-facing advice & 58 paired model--assignment units; 116 independently replayed episodes & 30 versus 9 certificates under the shared shield; not a search-versus-rule-advice comparison.\\
Native persistence & Eight executor qualifications; separate upstream demo/smoke & Qualified call paths, not planner--provider end-to-end effectiveness.\\
\bottomrule
\end{tabularx}
\end{table}

\paragraph{What changes between the model conditions.}
The execution layer checks every proposed action against the same compatible support. The evidence layer presents the same public task and observations. The advice condition additionally presents a planner's next-action recommendation; the model still proposes the action and the shared shield still determines admissibility. Zero newly introduced checked violations is therefore attributed to that common execution contract, whereas differences in completion and proposals describe the advice interface. The treatment bundles recommendation availability with its computation. Rule-derived advice and a legal-action-list interface were not model conditions, so neither recommendation grounding nor multi-step search is isolated as a causal mechanism.

\section{Related Work}
\subsection{Remote effects and durable execution}
Birrell and Nelson's RPC design makes the failure semantics of remote execution explicit~\cite{birrell1984rpc}. The end-to-end principle is relevant when delivery, duplication, and application correctness depend on facts outside a transport layer~\cite{saltzer1984endtoend}. These foundations explain why an HTTP error alone is not a negative effect certificate.

RIFL provides linearizability with durable request identity and completion records in cooperating storage systems~\cite{lee2015rifl}. Beldi supports fault-tolerant, transactional stateful serverless workflows through controlled logging and execution machinery~\cite{zhang2020beldi}. These are important positive examples of stronger execution guarantees when the required infrastructure is available. Our setting treats a third-party tool contract as an input: the controller cannot add a missing receiver-side atomic deduplication record merely by persisting a local request.

\subsection{Compensation and agent transactions}
Sagas decompose a long-lived transaction into subtransactions with compensating actions~\cite{garcia1987sagas}. The compensation is a separately executed operation, subject to permissions, dependencies, and its own failures. It need not reverse every external consequence. For example, removing a draft can restore a local object condition, whereas deleting a record cannot retract a message already delivered.

SagaLLM applies state management, independent validation, and recovery to multi-agent planning~\cite{chang2025sagallm}; we refer to its formally published PVLDB version. Atomix supports progress-aware transactional settlement, reversible external effects, and gating of irreversible effects~\cite{mohammadi2026atomix}. Its reported scope already includes ambiguous acknowledgments and partial irreversible settlement. Our decision object is the selection of evidence and continuation actions under a fixed, potentially incomplete provider contract. An original upstream implementation and an author-written reconstruction are distinct evaluation tracks.

\subsection{Recovery evaluation}
LIMBO varies models, harnesses, and tool contracts, and evaluates recovery conditions including verification, waiting, keys, and guards~\cite{li2026limbo}. It identifies the role of in-flight requests: a negative read does not justify a non-idempotent retry if an earlier request may still commit. Our common action checks incorporate this boundary. A strong combined verify--wait--key control is therefore necessary; blind retry and immediate stop serve only as diagnostic lower controls.

UndoBench uses paired normal and fault executions to distinguish ordinary task competence from recovery capability~\cite{sah2026undobench}. Effect-history checks are needed alongside final-state checks because an apparently correct final state can hide an extra irreversible action. We retain this distinction and compare all recorded assignments, rather than conditioning the main denominator on each method's own normal successes. Benchmark tasks and capabilities that an author extends are explicitly distinguished from the original workload.

The empirical question is whether joint evidence and continuation selection improves the recorded completion--cost relationship beyond these combinations. The claim requires equal observations, actions, authority, and resources. It cannot be supported by granting only the proposed controller an authoritative query or a compensating operation.

AgentSpec specifies runtime constraints and interventions such as stopping, obtaining user input, and replanning~\cite{wang2026agentspec}. We build on runtime enforcement as an established component and compare selection policies inside a common admissibility interface. DART studies semantic recoverability and rollback certificates for structured tool agents, including observed failures and timeouts, with dependency- and effect-aware rollback~\cite{yang2026dart}. Here, an opaque request may already have produced an external effect while returning no usable acknowledgment. The provider's evidence and termination capabilities determine whether continued execution or compensation is admissible; a local rollback certificate alone cannot establish that an unknown remote request has terminated.

\subsection{Planning under partial observation}
Partially observable planning provides the established framework for maintaining possible states and choosing actions contingent on observations~\cite{kaelbling1998pomdp}. Constrained POMDP methods address resource-limited decisions~\cite{isom2008cpomdp}. An expectation constraint is not a pathwise authorization guarantee. We separate the hard support-based action check from the weighted planning objective, and disclose the finite latent distribution used by the controlled model.

ReAct alternates reasoning and action to obtain useful observations~\cite{yao2023react}. Reflexion uses verbal feedback to improve subsequent decisions~\cite{shinn2023reflexion}. They motivate model-based recovery controls, but an agent's reflection is not a provider receipt. A same-model comparison must expose the same public evidence and charge the same tool, token, retry, and waiting costs. Scripted policies in the present controlled track are not counted as executions of either language-agent method.

\paragraph{Comparative scope.}
SagaLLM integrates persistent state and compensating recovery, while Atomix regulates settlement and irreversible release~\cite{chang2025sagallm,mohammadi2026atomix}. Their execution guarantees depend on their participating machinery; the present profile takes a fixed provider contract and a supplied transition model as inputs and evaluates remaining recovery choices within them. LIMBO and UndoBench inform capability and effect-history controls, but their qualifications here do not constitute information-matched model rankings. The measured added value is a recovery formulation, auditable conditional execution and a model-facing advice effect. The standard search, compensation and shielding components retain their existing research provenance.

\section{Tasks, Contracts, and Evidence}\label{sec:model}
\subsection{Approved obligations}
A task consists of a finite directed acyclic dependency graph, approved object-level goals, a deadline, and resource limits. An operation $i$ binds a target, a request intent, and an idempotency key when the provider supports one. The reference model uses distinct targets and distinct logical keys across operations. Repeated attempts of the same operation retain that binding. Tasks with shared mutable targets or cross-operation key aliasing require a richer provider model and are rejected by this profile.

Each operation has a goal value $g_i\in\{0,1\}$. Value one requires the intended object or effect to be present once; value zero requires an approved reversible staging effect to be absent. A retirement workflow can therefore require removal of an old staging configuration followed by activation of its replacement. This is different from changing the original goal after a failure. Removing the old configuration alone does not complete a task that also requires the new configuration.

The goal predicate is
\[
 G(s)=\bigwedge_i \left(x_i(s)=g_i\right),
\]
where $x_i$ is the operation's modeled object/effect count. Dependencies require the predecessor's approved goal value at the relevant effect boundary. The model distinguishes append effects from replacement effects. An additional append can violate an effect constraint even if a later repair yields an acceptable current object state.

The full history records both writes and compensations. Reversibility is an explicit capability, not an inference from a successful delete call. An irreversible notification-like append cannot be compensated by deleting a local record. Compensation of a predecessor is also restricted while a dependent effect exists or is still pending. These restrictions protect dependency consistency rather than treating a reverse traversal of the graph as sufficient authority.

\subsection{Provider capabilities}
Table~\ref{tab:contracts} identifies the relevant contract dimensions. The finite implementation instantiates these dimensions through explicit operation fields and validates their types, supported ranges, and bindings before execution. It does not discover authoritative guarantees from free-form model output.
\begin{table}[t]
\caption{Capabilities used in recovery decisions. A missing capability is not supplied by the local journal.}
\label{tab:contracts}
\begin{tabularx}{\linewidth}{@{}lY@{}}\toprule
Capability & Decision-relevant meaning\\\midrule
Deduplication & Stable intent key, validity interval, and atomic suppression of a previously accepted same-key write.\\
Query & Current-state versus attempt/history scope, authority, visibility delay, and supported operation binding.\\
Commit bound & A trustworthy upper bound for previously accepted requests; distinct from a bound on later attempts.\\
Final response & Whether a response certifies a settled attempt, rather than merely reception.\\
Compensation & Reversible effect class, reverse action, its deduplication behavior, and dependency conditions.\\
Authorization & Validity interval and whether authority is checked at acceptance or at effect commit.\\
Resources & Fees, calls, logical duration, and the business deadline.\\\bottomrule
\end{tabularx}
\end{table}

The provider model records an idempotency key at acceptance and suppresses subsequent requests while the declared key is valid. An expired key does not justify a retry. A deduplicated response can describe historical execution without proving that the current state satisfies the business goal; its observation scope is therefore retained.

Query authority and finality are separate. A query can report the current count accurately while an older request remains in flight. A delayed visibility path can return absence after a commit. In either case, absence is not automatically a certificate that another non-idempotent attempt is safe. For an initial request, the model can declare a commit bound. Subsequent attempts have their separately declared delay support. Conflating the two would prematurely certify the end of newly accepted work.

The model exposes acceptance-time and commit-time authorization. Under the former, a later withdrawal does not retroactively invalidate an already accepted request. Under the latter, prevention of an unauthorized late commit requires provider enforcement at commit or equivalent termination support. A local controller cannot promise to cancel an already accepted remote request when the provider offers no such capability.

\subsection{Public history and private outcomes}
Let $h_t$ contain the approved task and all public action/observation pairs up to time $t$. A latent state includes actual counts, in-flight effects, visibility history, accepted key records, and attempt positions. Beyond capability flags, the supplied model gives a finite set of initial states and, for each alternative, complete ordered write and compensation outcome sequences. Each outcome specifies success/failure, reply type, finite delay, and effect multiplicity. Attempt $k$ uses entry $\min(k,m-1)$ in a zero-indexed sequence of length $m$: after its final entry, later attempts repeat that entry. These candidate future sequences and commit times are visible to every active policy; only the realized alternative is hidden. They are strong known-transition inputs, not inferred provider guarantees or a learned fault model.

The default initializer takes the Cartesian product of operation alternatives and multiplies their locally normalized positive weights. This imposes cross-operation prior independence. Correlated failures require an externally supplied joint support, which the existing public-belief interface accepts; the primary 37-assignment enumeration uses the product prior. The weights describe author inputs rather than estimated incident probabilities. Safety still checks every retained alternative, regardless of its weight.

The belief support $B_t$ contains all latent states still compatible with $h_t$. Private scientific assignments select one member of the initial support for a recorded execution. The assignment identifier is never passed to the recovery selector. The environment's evidence export is accessible to collection and adjudication, not through the controller's tool interface.

For action $a_t$ and observation $o_{t+1}$, the update is
\[
 B_{t+1}=\{f(s,a_t):s\in B_t,\ \operatorname{obs}(s,a_t)=o_{t+1}\}.
\]
Here $f$ is the declared transition model, including elapsed time and due in-flight events. Equality covers the entire public observation: operation and request binding, time, status, count when supplied, scope, and finality. It does not compare only the return status.

Positive weights are normalized after observation filtering. They affect planning, but every surviving support member participates in safety checks. A contradictory observation that matches no member produces an explicit model-contradiction error. The controller does not continue with an empty support, substitute a convenient state, or use scientific truth to repair the contradiction.

\subsection{Safety and terminal completion}
The modeled safety conditions cover unauthorized effects, duplicate append effects, invalid dependency execution, and invalid compensation. Resource constraints cover all executed recovery actions. Initial or externally scheduled violations remain in the full history. They are reported separately from violations newly introduced by recovery; they are not removed from the task denominator.

For a finite-tail state that is initially safe, and a provider that obeys its declared transition model, an admissible action must preserve the checked conditions for every compatible state. On an initially unsafe task, the implementation checks that an action does not add modeled violations or increase the covered abstract unknown-pending hazards. Evidence acquisition and isolation can remain legitimate even when the original task can no longer be counted as safely completed.

Completion requires more than $G(s_t)$. Define $C(h_t)$ to mean that every state in $B_t$ satisfies the approved goal, resource limits, and terminal safety. For finite tails, terminal safety requires the goal and effect constraints to remain satisfied after every already-accepted future effect in the declared serialization order. Equal final counts alone are insufficient: a pending compensation followed by a restoring write can destroy the goal in between. Equal-time events use the same dependency and stable-identifier ordering as execution. A same-valued replacement can be harmless.

An initial pending event with \texttt{commitAt: null} is instead an abstract unknown-delay sentinel. The executor never advances that sentinel into a concrete commit, even after waiting. It stands for missing settlement knowledge, not an empirical event that is guaranteed never to occur and not an enumeration of every arbitrarily late finite time. The separate containment checks cover specified duplicate, dependency, and compensation hazards; any remaining sentinel prevents certification. This abstraction is not a complete ordering model for arbitrary unknown remote delays.

Let $\Pi_{\mathrm{adm}}$ contain policies measurable with respect to public history whose selected actions pass the common action check. The declared initial weights induce $b_0$, and the controlled objective is a publicly supported safe completion at a stopping time $\tau$ within deadline $T$:
\[
 \max_{\pi\in\Pi_{\mathrm{adm}}} \Pr_{b_0,\pi}\!\left[\exists\tau\leq T:\ G(s_\tau)\land C(h_\tau)\land I(H_{\mathrm{terminal}})\right].
\]
$I$ evaluates the explicit finite effect history, including late effects beyond the task cutoff. Unknown sentinels remain unresolved rather than being labeled through a fabricated concrete tail. In an actual service, a finite observation window alone is not a termination proof. A trusted bound or a separately qualified harmlessness argument is necessary. Handoff and compensated exit remain distinct outcomes; neither is automatically counted as the original goal's completion.

\section{Evidence and Continuation Planning}\label{sec:method}
\subsection{Common action admissibility}
The action interface contains query, write, wait, compensate, and handoff. Each action consumes declared fee, time, and call resources. Waiting is charged under the task's public resource model, and handoff consumes resources without resolving the goal by definition. All primary controls use the same action interface and admissibility procedure.

For each candidate action, the procedure enumerates the full current support, evaluates the declared transition, and checks the remaining resource limits. It settles finite pending effects for both the no-new-action reference and the action's resulting state. An action is rejected if it introduces an additional modeled violation. For unknown-delay sentinels, a separate abstract check rejects increases in the covered duplicate, dependency, or compensation hazards; it does not concretely replay arbitrary late commits. The safety check does not sample the support or omit a low-weight state.

This procedure gives different answers under different contracts. With a valid atomic same-key receiver contract, a repeat attempt can be allowed when it cannot append another effect. Without that contract, an uncertain append remains blocked unless the public model and evidence rule out both a prior commit and a harmful future commit. A replacement can be admissible under the implemented same-target, same-intent model, but this is not permission to overwrite a different business intent.

Compensation is checked as an ordinary externally effectful action. Its target must be reversible and approved, its dependencies must allow removal, and its possible delayed result must be included in the safety assessment. A timeout during compensation therefore creates another evidence problem, not a successful rollback marker. Repeated compensation depends on its own deduplication contract.

\subsection{Observation-contingent search}
The planner builds branches for complete public observations. For action $a$ at belief $b$, all model states predicting the same observation form a branch, and their positive masses give its probability. The planner evaluates subsequent admissible actions within each resulting belief, building on standard partially observable planning~\cite{kaelbling1998pomdp}.

The implemented value is lexicographic. It prioritizes the projected probability of certified completion, then expected satisfaction of approved obligations, then negative recovery cost. The final cost component uses the task's declared fee plus a fixed small duration tie-breaker. This tie-breaker is a controlled implementation choice, not a conversion of money into a universal value of human time. Actual fees, calls, duration, and compensation counts are reported separately.

An iteratively deepened expectimax search evaluates the observation branches. It caches complete belief/depth subproblems and orders actions by public one-step values. Only a layer in which all root actions and their observation branches have been backed up at the same depth can replace the previous selection. Diagnostics retain each completed root action's full vector $(\text{completion probability},\text{goal progress},\text{negative cost})$. A partially explored deeper layer is recorded but cannot compare its deeply searched action against another action's immediate leaf value. A node budget limits optimization effort; it never limits the support checked by admissibility. With no completed informative layer, selection uses the capability-rule fallback. Exactness is restricted to the completed horizon and supplied transition model and prior. Receding-horizon execution does not imply globally optimal recovery or guaranteed completion under a finite budget.

Resource-constrained POMDP shielding already combines hard action exclusion with planning~\cite{ajdarow2023shielding}. Our use of this separation is explicit. The application-specific content is the provider evidence semantics, task obligations, in-flight effect tail, and experimental comparison of recovery decisions under those semantics.

\subsection{Matched controls}
The capability-rule control can continue independent legal obligations, attempt valid same-key recovery, obtain informative queries, wait for possible evidence, traverse a finite safe retry sequence, compensate in dependency-safe order, and hand off. It receives the same public latent support as the planner. It is not deprived of a query or inverse operation to produce a favorable comparison.

The verify--wait control uses the same admissibility checks and actions, but prioritizes informative queries before continued execution. It is an author implementation motivated by the recovery combinations in LIMBO~\cite{li2026limbo}, not an execution of an unavailable original artifact. The greedy control ranks safe actions using one-step completion, obligation progress, information change, and cost. It has the same prediction model but does not build a multi-step observation tree. Immediate stop is a diagnostic control and does not report success on an unresolved task.

The controls' fixed priorities and tie-breaking rules are recorded in source before each collection. Both greedy and joint selection use the complete capability rules when their evaluated values provide no completion or obligation progress. That fallback is part of the implementation, not evidence of a separate planning gain. Safe retry opportunities are not filtered out merely because the immediately predicted attempt fails: a finite sequence can contain a known failed attempt followed by an allowed successful retry.

\subsection{Execution and handoff}
After selection, the controller executes one public action, records its observation, updates the full support, and replans. A completion check can stop execution only when all supported states satisfy the terminal certificate. Otherwise, the controller returns an unresolved condition, an explicit handoff, a model contradiction, or a recorded resource/decision-limit outcome.

A practical handoff should contain the bound logical requests, known effects, unresolved alternatives, remaining obligations, permitted next actions, and resources already consumed. The controlled interface records the handoff decision and its cost. It does not simulate an operator as if the operator had completed the task, nor does it assign an operator access to private scientific state.

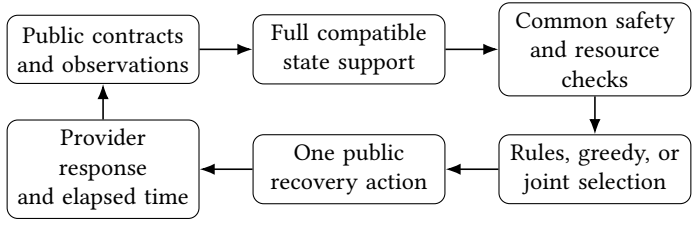
\begin{figure}[t]
\centering
\begin{tikzpicture}[node distance=5mm and 7mm,box/.style={draw,rounded corners,align=center,font=\small,text width=2.3cm,minimum height=9mm},arr/.style={-{Latex},semithick}]
\node[box] (e) {Public contracts\\and observations};
\node[box,right=of e] (b) {Full compatible\\state support};
\node[box,right=of b] (s) {Common safety\\and resource checks};
\node[box,below=of s] (p) {Rules, greedy, or\\joint selection};
\node[box,left=of p] (a) {One public\\recovery action};
\node[box,left=of a] (o) {Provider response\\and elapsed time};
\draw[arr] (e.east)--(b.west);
\draw[arr] (b.east)--(s.west);
\draw[arr] (s.south)--(p.north);
\draw[arr] (p.west)--(a.east);
\draw[arr] (a.west)--(o.east);
\draw[arr] (o.north)--(e.south);
\end{tikzpicture}
\caption{The primary policies differ at selection, not at the safety boundary. The private effect ledger is used only by adjudication and is not an arrow into the controller.}
\Description{A six-node loop from public contracts through compatible support and shared safety checks, policy selection, one recovery action, and the resulting public observation.}
\label{fig:control}
\end{figure}

\section{Implementation and Conditional Properties}
\subsection{Finite reference profile}
The reference implementation uses JavaScript modules for the provider model, public-belief update, and recovery policies. A separate Python implementation reconstructs observations, effects, costs, terminal state, and the public certificate. Both implementations operate on synthetic requests and logical time. No customer schema, real financial transfer, or sent email is represented as a measured deployment.

Schema validation rejects cycles, cross-operation target/key aliasing, unsupported fields, invalid assignments, and unsafe integer arithmetic. A task has at most twelve operations and at most 4,096 initial support combinations. These are declared implementation limits. The study's templates remain within them; the limits are not evidence of scalability for large business workflows.

The environment and controller are separate interfaces. Collection can obtain the private evidence export after an episode. The recovery module imports only the declared finite transition model, not a live environment instance or its evidence export. Supplying the public model is part of the controlled assumption: it does not mean that arbitrary provider behavior has been inferred from a timeout.

\subsection{Support retention}
\begin{proposition}[Retention under a correct finite contract]
Suppose the actual initial state belongs to the declared support, the provider follows the declared transition/observation model, and the controller retains every state matching a complete observation. The actual state remains in the support after every successful update.
\end{proposition}
\begin{proof}
Initially the conclusion holds by the support assumption. If the actual state belongs to the current support, its declared transition produces the actual next state and public observation. Exact observation filtering therefore retains that transition. Applying the argument to each action proves the claim. A contradictory observation triggers an error rather than a completed update.
\end{proof}
The proposition characterizes support retention by the implemented observation filter under the stated contract. Initial-support coverage and provider adherence are its explicit premises.

\subsection{Action and completion conditions}
\begin{proposition}[Conditional preservation and completion]
Under the retention assumptions for the explicit finite transition profile, an action checked against every support member does not introduce any violation covered by the implemented action check. If the initial state is safe including its accepted finite effect tail, repeated admissible actions preserve those checked conditions. If a completion certificate holds for every support member and every accepted finite-tail prefix, it holds for the actual supported execution. Separately, the unknown-sentinel abstraction rejects actions that increase its covered hazard counts and refuses completion while a sentinel remains.
\end{proposition}
\begin{proof}
Retention includes the actual execution in each checked support. The action check therefore includes the actual transition and its modeled pending tail. An action that adds a covered violation in that transition would be rejected. Starting from a safe settled initial state, induction preserves that condition. A prefix without a violation is insufficient if an accepted request is already scheduled to violate a condition later. The same membership argument applies to the universal completion check.
\end{proof}
The finite-state membership argument is conditional on transition-model coverage. The unknown-sentinel check is an abstract containment rule, not a proof over all concrete commit times and interleavings. Neither grants receiver-side cancellation, distributed consensus, or an omitted provider capability. On an initially unsafe history, the relevant implementation condition is containment of new checked violations; the prior violation remains an absolute task failure.

For finite delays, the model settles all already accepted events when assessing the future tail. A residual unknown-delay sentinel prevents terminal certification and remains distinct from a concrete scheduled commit. A non-idempotent retry is inadmissible when the same public history permits both a prior effect and an in-flight future effect that could duplicate it. This action boundary applies established remote-outcome reasoning and the in-flight analysis studied by LIMBO~\cite{birrell1984rpc,li2026limbo}.

\subsection{Native persistence qualification}
An existing single-host durable executor and SQLite store are reused for a separate native qualification. Its public source repository is \url{https://github.com/VancouverCanada/openport}; the qualification binds the locally tested build by file hashes rather than assuming its package version identifies every tested byte. The qualification exercises successful persistence, replay without invoking the handler again, reopened receipt reconstruction, wrong-request and altered-result rejection, missing evidence, and ambiguous delivery containment. An acknowledgment-loss handler creates a synthetic effect before throwing; the default executor retains manual review and does not automatically invoke that handler again.

This is an actual call path through the existing public modules, rather than a renamed reimplementation of their state machine. It is also a qualification track, not a claim that the joint planner has been deployed through a third-party service or that every manual-review state has an automated native reconciliation path. The controlled decision study and native persistence checks retain separate evidence and denominators.

\section{Study Design}\label{sec:study}
\subsection{Recorded workload and denominator}
The controlled workload is a fixed set of authored mechanism scenarios. It contains \MechanismTemplates{} public templates and enumerates all \MechanismAssignments{} combinations of their declared initial latent support. Each combination is executed by the five policies, yielding \MechanismRecords{} records. A template is a business/contract scenario, not a new independent task for every change of an entity name. The workload is development and mechanism evidence, not a random sample of enterprise incidents.

The workload covers dependent billing-like records, a stock dependency diamond, a file fork, independent obligations, and item-level partial batches. Additional templates isolate valid and expired keys, delayed reads, pending requests, absent queries, late effects beyond a deadline, ambiguous compensation, an approved retirement-and-activation goal, a safe retry sequence, initial duplication hazards, two authorization timings, and a restricted budget. The batch profile exposes individual operations. It does not represent an opaque all-in-one call with arbitrary unknown subsets.

All methods receive the same public task, initial observation, finite support, action set, and limits. Private assignments choose the realized state without changing the public inputs. The collection runner passes only the public history and belief into selection. It retains an episode even when an existing violation makes safe completion impossible or when no safe continuation remains. Immediate stop is a diagnostic control; comparisons of planning choices use the shared-check policies.

\subsection{Freezing and collection}
The collection protocol is written before the first episode. It binds the exact case list, policies, resource model, maximum number of decisions, search configuration, runtime, and hashes of the scientific source closure. Evidence is append-only within a new numbered collection directory. A record contains actions, complete observations, effects, resource totals, stop reason, errors, and public search diagnostics. The private fields are for adjudication, not policy input.

The bounded planner uses a maximum horizon of six decisions and a budget of 128 non-leaf belief expansions across iterative-deepening layers. This is not a cap on transition evaluations, full-support checks, or wall computation. It does not remove states from an admissibility check. Every policy has at most 96 decisions and must respect the task's tighter fee, logical-action, and deadline constraints. Search truncations and the completed horizon are retained. Logical waits model decision consequences and are not measured wall-clock service latency.

This development process includes an initial recorded planner qualification and, when necessary, a separately recorded method revision. The initial record and its source archive are preserved. Results from a revision on the same templates remain development evidence; they are not relabeled as a blinded or held-out confirmatory test.

\subsection{Independent adjudication}
The checker first validates input/source bindings, the complete case--policy denominator, record order, hash chain, and recorded errors. It then re-executes each deterministic public policy and environment trajectory. Wall-clock collector duration is excluded from deterministic equality; all scientific actions and observations are included.

A separate Python program replays the assigned provider state and actual actions. It reconstructs due effects, complete observations, resource totals, dependencies, authorization, and final counts. It also maintains the complete public support to check a completion claim. The program imports no JavaScript policy or goal predicate. Validation tests reject forged completion, changed costs, lost late effects, malformed state bindings, and inconsistent observations. Independent implementation reduces shared-label errors; it is not a claim of independence from the common written specification.

The primary descriptive count is certified safe completion within the task deadline, including all explicit finite-tail prefixes. We also report private goal satisfaction at the task deadline, existing absolute violations, newly introduced violations, handoff, unresolved outcomes, and all costs. A deadline goal can be true while its public certificate is false. Unresolved counts include handoff and therefore overlap with that category. Unknown-delay sentinels remain unresolved; their eventual concrete effects are not claimed to have been adjudicated.

\subsection{Analysis and native tracks}
Counts and totals describe the enumerated finite set. Templates with more declared states contribute more assignments to the total; this enumeration weight is not an estimated deployment prevalence. No population-level confidence interval, statistical significance test, or production failure probability is inferred from these authored templates. Fees are declared synthetic units, logical actions include waits and handoff, and logical durations include waits. These quantities are reported separately from external RPC count, monetary price, and wall latency.

The separate post-review development audit parameterizes write versus compensation, zero versus delayed outcomes, dependency presence, equal versus unequal priors, failed-first attempt sequences, and initial finite pending effects. An additional externally supplied correlated joint support reuses the public-belief interface. Python exhaustively evaluates every admissible root action at horizons one, two, and three, returning all three objective components and tied optima. JavaScript complete layers are compared componentwise; budget sweeps additionally check that incomplete layers do not affect selection and that fallback agrees with the same rules. This audit shares the authored specification, not the JavaScript decision code, and is stored separately from the 37-case denominator. A deliberately mis-specified attempt model tests observable contradiction handling; it does not estimate the frequency of model errors or prove detection of unobservable errors.

The cost diagnostic preserves every assignment--policy fee, action count, logical duration, and measured selection wall time from the original and current 185-row collections. It reports per-assignment and per-template Pareto sets and a symbolic fee/compute break-even, without assigning market prices. Selection wall time is a compute proxy, not measured CPU billing or provider latency; correlated host load is not controlled by these single-run timings.

The separate native durable-executor qualification exercises eight persistence and evidence cases. The Atomix native demo and an UndoBench scripted smoke run use pinned original repositories and isolated working directories with no model credentials. Their effects and reports are retained separately. The UndoBench script can access a benchmark-internal plan, so it is a pipeline qualification, not an information-matched model baseline. A successful command exit is distinguished from the benchmark's own task verdict.

\subsection{Local-model qualification and assignment-held-out confirmation}
\label{sec:model-study}
The separate qualification pilot evaluates eight authored profiles with one latent assignment per profile, two strategies and two local model bundles, yielding 32 reported cells. Four duplicate smoke cells remain outside that denominator. The evidence-prompt condition exposes public recovery evidence and the action schema; planning assistance additionally supplies the bounded planner's recommended next action. The model proposes an action in both conditions. A common deterministic shield checks admissibility before execution, and a separately implemented Python adjudicator determines terminal outcomes. Both bundles use LM Studio SDK 2.0.0, temperature 0, a requested reasoning budget of 128 and a qualification output cap of 512 tokens.

The subsequent source-bound preregistration fixes the exhaustive complement of the eight pilot assignments within the 37-assignment mechanism frame. All 29 remaining assignments are retained across two models and two strategies: \RecEpisodes{} episodes and \RecPairs{} matched model-by-assignment pairs. The 19 public task families, planner, provider model and shield were used during development. The held-out axis is the latent assignment. Family holdout and independent-runtime replication are additional evaluations.

The confirmatory configuration freezes \path{openai/gpt-oss-20b} and \path{zai-org/glm-4.7-flash}, temperature 0, an input cap of 2,048 tokens, a 96-token output cap, a requested reasoning budget of 128, and at most eight model proposals per episode. Failed proposals consume the common allowance. The maximum is 928 outcome calls, plus at most four preflight calls. Every included record retains prompts, responses, telemetry, admissibility decisions and environment effects. Both model runs completed their entire frozen frame without interface errors or output truncation; Python replay validates all 116 trajectories.

The registered primary endpoint is certified safe completion. Planner-minus-evidence differences are summed within public task family across assignments and both models. The two-sided exact sign-flip test operates on nonzero family sums at $\alpha=0.05$; its inferential interpretation assumes sign symmetry under the null. An exact McNemar test over the 58 pairs is a registered sensitivity analysis. Counts describe the authored frame rather than production incident prevalence. Tokens, proposals and planner computation remain separate resource measurements.

\section{Recorded Results}\label{sec:results}
\subsection{Completion and resources}
Table~\ref{tab:policies} reports the independently adjudicated collection after the common-layer planning revision. Verify--wait rules, capability rules, greedy selection, and joint planning each achieve 28 certified safe completions out of 37 assignments. Immediate stop achieves two, corresponding to goals already certifiable at the start. The active-policy comparison therefore concerns resource use and decision behavior at the same observed completion count, not a completion-rate improvement over strong rules.

\begin{table}[htbp]
\caption{Complete finite enumeration after the planning revision. Goal is private goal satisfaction at the deadline; new violations are episodes with at least one recovery-introduced violation. Fees and time are synthetic units. Logical actions include waiting and handoff.}
\label{tab:policies}
\small
\begin{tabularx}{\linewidth}{@{}Yrrrrrrr@{}}\toprule
Policy & Safe & Goal & New & Handoff & Fee & Actions & Time \\ \midrule
Immediate stop & 2/37 & 9 & 0 & 0 & 0 & 0 & 0 \\
Verify--wait rules & 28/37 & 30 & 0 & 6 & 191 & 106 & 109 \\
Capability rules & 28/37 & 30 & 0 & 6 & 192 & 105 & 108 \\
Greedy selection & 28/37 & 30 & 0 & 6 & 195 & 94 & 113 \\
Joint planning & 28/37 & 30 & 0 & 6 & 188 & 103 & 108 \\
\bottomrule
\end{tabularx}
\end{table}

Joint planning incurs 188 fee units and 103 logical actions, compared with 192 and 105 for capability rules. Both have 108 aggregate logical time units. Verify--wait uses 191 fee units and 106 actions; greedy uses 195 and 94. Greedy has fewer logical actions than joint planning but higher fees and time. Thus, no arm dominates all reported resource measures. These are descriptive totals for the frozen enumeration; the fee difference is not a population-level efficiency claim.

Each active arm has six handoffs and three other unresolved assignments, in addition to its 28 completions. The unresolved total of nine includes those six handoffs. All assignments are retained, including tasks whose initial effects already violate an invariant. Three assignments per arm have an absolute historical violation; none has a recovery-introduced violation or a false completion claim. The common action check is shared, so the zero-new-violation result is not attributed specifically to joint selection.

\subsection{Goal satisfaction and evidence}
Each active arm has 30 assignments whose private goal is satisfied at the deadline but only 28 safe certified completions. The gap has two different causes. In a no-query case, the assigned effect has occurred, but public evidence still permits an absent or pending alternative. The controller cannot turn adjudicator knowledge into a certificate. In another case, the provider does not enforce commit-time authorization: the effect occurs after authority has expired and is an absolute violation. The second case is unsafe, rather than merely missing evidence.

The late-effect template preserves a commit beyond the task deadline. The deadline snapshot and full effect tail are both adjudicated. The valid-key and expired-key templates exercise different allowed recovery sets; they do not establish coverage of every key-expiry race. The item-level partial batch permits continuation of missing items while preserving completed items. Its result does not imply equivalent recovery for an opaque, coarse-grained batch.

The retirement-and-activation templates execute compensation and its ambiguous outcome as part of a fixed approved goal. Removing the staging object is insufficient without activating the replacement. Conversely, the template named ambiguous compensation begins with a one-valued goal already satisfied in both support states; its zero-action episodes are initial-certificate controls, not additional compensation-execution evidence. Coverage follows the recorded actions rather than the template name.

\subsection{Retained development revision}
The first source-bound collection is preserved with its original planner. That sequential depth-first implementation produced 27 safe completions, 218 fee units, 130 logical actions, and 140 logical time units for joint planning, while capability rules produced 28, 192, 105, and 108. In the missed retirement assignment, the planner repeatedly waited before initiating compensation and failed to activate the replacement by deadline.

The cause was unequal evaluation depth under a shared expansion budget. An early root action could obtain a deep continuation value while later actions were compared using immediate leaves after budget exhaustion. Replanning repeated that comparison. The revised implementation only selects from complete common-depth layers and keeps unfinished layers diagnostic. It retains the same 128-expansion limit, maximum horizon, support, safety check, and strong-rule policies. Generic delayed-write and delayed-compensation regressions vary identifiers, operation ordering, and uninformative actions rather than recognizing a task name.

The second collection reuses the development templates, so its improved planner behavior is a development revision result, not an independent confirmatory gain. The revised planner matches the strong controls' completion count and has the fee/action totals in Table~\ref{tab:policies}. Original negative results, causes, source bytes, and denominators remain reproducible.

\subsection{Stable tails and measured computation}
\label{sec:mechanism-cost}
The third collection retains the same 37 assignments and five methods after a certificate revision. Certification checks every future effect prefix, including equal-time serialization, rather than comparing only current and settled counts. Independent hand-calculated regressions reject a compensation followed by a restoring write at both distinct and equal times; a finite same-valued replacement remains certifiable, while an unknown-delay sentinel remains unresolved. The current fourth collection adds per-root-action full objective-vector diagnostics without changing action selection. It is fully recollected and independently replayed under its own source bindings. All prior collections and their source bytes remain intact. Completion and synthetic resource totals are unchanged; these are corrected/instrumented development replications, not additional independent samples.

\begin{table}[htbp]
\caption{Measured computation in the complete source-bound enumeration. Episode and selection totals and per-decision median/p95 are host wall times in milliseconds, not provider latency. Cutoff counts truncated search decisions; fallback counts joint decisions selected by the common capability rules. Shared-host timings are descriptive and not monetary totals.}
\label{tab:compute}
\small
\begin{tabularx}{\linewidth}{@{}Yrrrrrr@{}}\toprule
Policy & Episode & Selection & Median & p95 & Cutoff & Fallback \\ \midrule
Immediate stop & 35.82 & 24.36 & 0.37 & 2.41 & 0 & 0 \\
Verify--wait rules & 173.81 & 134.39 & 0.73 & 3.86 & 0 & 0 \\
Capability rules & 180.83 & 139.83 & 0.70 & 3.55 & 0 & 0 \\
Greedy selection & 168.76 & 132.66 & 0.73 & 5.01 & 0 & 0 \\
Joint planning & 9576.16 & 9531.66 & 24.78 & 288.53 & 90 & 20 \\
\bottomrule
\end{tabularx}
\end{table}

Joint selection uses 83 planner-selected actions and 20 rule-fallback actions. Ninety decisions truncate the search, and thirteen finish the maximum six-step horizon. Its measured selection time totals \JointSelectionSeconds{} seconds, versus \RulesSelectionSeconds{} seconds for capability rules in the current collection. Selection p95 is \JointSelectionPninetyfive{} milliseconds versus \RulesSelectionPninetyfive{} milliseconds. These are shared-host measurements with implementation and load effects, not service latency, CPU billing, or API prices.

The persistent cost diagnostic retains all 185 assignment--policy rows of the original third collection and all 185 rows of the current fourth collection. It reports each cost component and Pareto membership per assignment and public template. For the third collection, joint planning saves four fee units and two actions, at \OriginalExtraSeconds{} additional measured selection seconds. With symbolic prices $p_F$ per synthetic fee unit and $p_S$ per selection second, ignoring other deployment terms, the measured break-even condition is
\[
 4p_F>\OriginalExtraSeconds\,p_S,
 \qquad p_F/p_S>\OriginalBreakEven\ \text{seconds per fee unit}.
\]
No market price is assigned to either unit. Calls, logical latency, integration, and uncertainty in shared-host timings remain separate terms. Templates with zero fee savings have no positive fee/compute trade-off of this form; negative savings are retained. Thus, the four-unit fee reduction does not demonstrate a monetary net benefit or a Pareto improvement over matched capability rules when compute is included. Table~\ref{tab:task-cost} discloses where the savings and overhead occur.

\subsection{Independent development diagnostics}
The original six-pair small reference remains archived. The expanded post-review audit checks \DiagnosticPairs{} task--horizon pairs from \DiagnosticTasks{} separately authored parameterizations at horizons one, two, and three. All \DiagnosticRootValues{} admissible root-action values match the independently enumerated three-component Python vectors, with selected actions checked against tied optimal sets when the planner selects. The \DiagnosticBudgetPairs{} budget pairs exercise complete common layers, \DiagnosticPartialLayers{} incomplete deeper layers, and \DiagnosticFallbacks{} rule selections. Their completed vectors and selection basis agree with the exhaustive reference or the documented fallback in every pair.

The parameterizations cover finite action delay, initial finite pending effects, approved zero-count compensation, dependencies, unequal weights, failed-first attempt sequences, and externally supplied correlated support. A deliberately excluded successful outcome produces \texttt{BELIEF\_CONTRADICTION}, not an empty belief or another execution decision; an unknown sentinel stays pending and uncertifiable after waiting. These checks establish implementation agreement for this qualified small domain, not detection of unobservable model errors. They were designed after review, are development diagnostics rather than held-out tests, and do not increase the 37-assignment empirical denominator or establish global recovery optimality.

\subsection{Native qualifications}
The native durable executor passes eight local persistence/evidence checks. Its lost-acknowledgment handler produces one synthetic effect before raising an error; a later invocation remains in manual review and does not execute the handler again. These checks establish that particular native containment path, not automated task completion by the joint planner.

The pinned Atomix demo executes its winning branch and reports a clean aborted partial branch. The pinned UndoBench smoke command also completes, but its own report records one passing normal condition out of two pairs, zero passing fault conditions, and one pair with duplicate effects. The run is retained as an original scripted pipeline qualification. Its script has benchmark-internal plan access, and its verdicts are not included in Table~\ref{tab:policies} or used to rank the joint method against UndoBench.

\subsection{Local-model qualification results}
Table~\ref{tab:local-model-qualification} reports the separate authored
qualification pilot. Evidence-only prompting yields two certified safe
completions in 16 cells, while planning-assisted prompting yields seven in 16.
Private goal satisfaction is eight versus twelve cells, respectively.
Planning assistance uses 79 model proposals rather than 114 and receives 51
rejections rather than 111. It also produces 28 executed admissible actions,
compared with three under evidence-only prompting.

% Proposed table generated from:
% Generated from the archived local-model qualification summary.
\begin{table}[htbp]
\caption{Authored local-model qualification pilot. Safe is certified safe
completion; goal is private goal satisfaction at the deadline. Attempts count
model proposals, and rejected counts proposals blocked before execution. The
32 cells are development qualification evidence, not confirmatory population
evidence.}
\label{tab:local-model-qualification}
\small
\begin{tabularx}{\linewidth}{@{}Yrrrrrrr@{}}\toprule
Strategy & $N$ & Safe & Goal & Certified & Handoff & Attempts & Rejected \\
\midrule
Evidence prompt & 16 & 2 & 8 & 2 & 0 & 114 & 111 \\
Planning assisted & 16 & 7 & 12 & 7 & 2 & 79 & 51 \\
\bottomrule
\end{tabularx}
\end{table}

The descriptive differences are concentrated in stale-query and partial-batch
profiles. Planning assistance certifies both item-level partial-batch cells,
both coarse-batch proxy cells, and one of two stale-query cells; evidence-only
prompting certifies none of those cells. Neither condition certifies the
terminal-query, no-terminal-query, compensation-limited, or authorization
profiles across both models. The two handoffs under planning assistance occur
in the no-terminal-query and authorization profiles.

Across both strategies, \texttt{gpt-oss-20b} records five certified safe
completions in 16 cells and \texttt{glm-4.7-flash} records four. The experiment
does not isolate model identity as a causal variable: both bundles share LM
Studio, each profile has one authored hidden assignment, and there are no
independent repeated samples. We therefore report counts and failure patterns
without significance tests or a general claim that planner advice improves
model recovery.

\subsection{Assignment-held-out confirmatory results}
\label{sec:model-confirmation}
\begin{table}[t]
\caption{Assignment-held-out confirmation with a common admissibility shield. Each strategy has 58 model-by-assignment episodes; public task families were used in development.}
\label{tab:confirmatory-recovery}\centering\small
\begin{tabular}{@{}lrrrr@{}}\toprule
Strategy & Certified & Actual goal & Handoff & Attempts \\\midrule
Evidence prompt & 9/58 & 15/58 & 0/58 & 399 \\
Planning assistance & 30/58 & 33/58 & 4/58 & 272 \\
\bottomrule\end{tabular}
\end{table}

Planning assistance certifies \RecPlanner{} of \RecPairs{} paired episodes (\RecPlannerPct\%), compared with \RecEvidence{} (\RecEvidencePct\%) under evidence prompting. Nine pairs are completed by both strategies, 28 by neither, 21 only with planning assistance and zero only with evidence prompting. The paired difference is \RecDifference{} percentage points. Eight of 19 family sums are nonzero and all favor planning assistance; the registered exact test gives $p=\RecP$. The McNemar sensitivity result is $9.54\times10^{-7}$ and does not replace the family-level primary test.

The actual deadline goal is satisfied in 33 planning-assisted and 15 evidence-prompt episodes (Table~\ref{tab:confirmatory-recovery}). These counts differ from public certification, which also requires the recorded safety and finite-tail conditions. Four planning-assisted episodes hand off; these remain distinct from completion. Each model contributes 29 matched assignments and 58 episodes across both strategies. The per-condition breakdown below avoids using pooled model totals as an advice-effect comparison.

\begin{table}[t]
\caption{Recorded model resources and rejected proposals in the confirmatory study. Prompt and completion totals are observed tokens; reasoning is reported separately and is not added again to completion tokens. Planner computation is additional CPU work.}
\label{tab:confirmatory-resources}\centering\small
\begin{tabular}{@{}lrrrr@{}}\toprule
Strategy & Unsafe proposals & False claims & Prompt tok. & Completion tok. \\\midrule
Evidence prompt & 317 & 218 & 273,983 & 8,310 \\
Planning assistance & 123 & 76 & 196,822 & 5,686 \\
\bottomrule\end{tabular}
\end{table}

Planning assistance uses 272 proposals versus 399, and 202,508 observed prompt-plus-completion tokens versus 282,293. Rejected unsafe proposals fall from 317 to 123 and false completion claims from 218 to 76. These claims are recorded proposals, not certified outcomes. Reasoning telemetry is reported separately (2,892 versus 4,090 tokens) without adding it twice to completion totals. Search advice consumes additional CPU work; fewer model tokens do not establish lower total monetary or deployment cost.

The completion advantage concerns the two evaluated local bundles, common shield and supplied known provider model. The deterministic strong controls tie at 28/37 in their separate study. The model comparison establishes a benefit of exposing planner advice relative to the matched evidence-prompt interface. Its batch conditions use item-level proxies, and both bundles run through the same local runtime with authored provider transitions. These conditions define the execution frame of the confirmation.

\subsection{Recorded heterogeneity and resource interpretation}
The following descriptions were added after the registered outcome was available, using archived episode records and the separate oracle labels only. They do not add assignments, rerun models or replace the clustered primary test.
% Derived from archived records and independent oracle labels only.
\begin{table}[t]
\caption{Post-hoc confirmation breakdown. Each row reports paired units; counts are certified completion. Families listed separately have a nonzero net difference; the remaining eleven are aggregated without omission.}\label{tab:recovery-strata}
\centering\small
\begin{tabularx}{\linewidth}{Yrrr}
\toprule
Stratum & Pairs & Evidence & Advice\\
\midrule
gpt-\allowbreak{}oss-\allowbreak{}20b & 29 & 5 & 15\\
glm-\allowbreak{}4.7-\allowbreak{}flash & 29 & 4 & 15\\
expired-\allowbreak{}key-\allowbreak{}recovery & 4 & 0 & 4\\
billing-\allowbreak{}chain & 4 & 0 & 2\\
approval-\allowbreak{}independent & 4 & 0 & 3\\
retirement-\allowbreak{}already-\allowbreak{}inflight & 4 & 0 & 1\\
authorization-\allowbreak{}at-\allowbreak{}acceptance & 2 & 0 & 1\\
stock-\allowbreak{}diamond & 6 & 0 & 6\\
file-\allowbreak{}fork & 2 & 0 & 2\\
batch-\allowbreak{}items & 2 & 0 & 2\\
Other eleven families & 30 & 9 & 9\\
\bottomrule
\end{tabularx}
\end{table}
\newcommand{\RecLooMin}{28.8}
\newcommand{\RecLooMax}{38.9}
\newcommand{\RecTokenReduction}{28.3}

Table~\ref{tab:recovery-strata} reports 15/29 versus 5/29 certificates for gpt-oss-20b and 15/29 versus 4/29 for GLM-4.7-Flash. The eight nonzero family contrasts account for all 21 net additional completions; stock-diamond contributes six and expired-key recovery four. Removing any whole family retains a positive descriptive difference, ranging from \RecLooMin{} to \RecLooMax{} percentage points. The other eleven families contain 30 pairs and tie at 9/30. Thus advice is useful in specific tested recovery structures, not uniformly across every capability profile. No family was removed from the registered analysis.

Four episodes in each condition require no model call and are already certifiable at initialization; they remain in the 58-pair denominator. Advice executes 101 admissible actions versus 34 under evidence prompting, while reducing proposals from 399 to 272. Rejected unsafe proposals are attempts, not executed violations, and repeated false completion claims are counted per proposal rather than per episode. These quantities describe how the interface reaches valid progress, not a change in the shared authority or safety contract.

Observed prompt-plus-completion usage falls by \RecTokenReduction{}\%, but the advice adds planning computation. The deterministic strong-rule study supplies the explicit fee/compute trade-off in Section~\ref{sec:mechanism-cost}; its timing is not a substitute for a per-episode search-time measure in the model study. Total deployment economics must retain provider actions, waiting, model usage, search computation and integration as separate charged categories. Where a category is not measured for this model frame, no total monetary saving is inferred.

\section{Interpretation and Applicability}
\subsection{Operational meaning}
Recovery should expose what is still owed, what is known to have occurred, and which evidence would permit another action. A blanket timeout retry conceals these distinctions. A blanket stop can also leave independent approved obligations unfinished. The reference interface makes continued execution, safe waiting, compensation, and unresolved handoff comparable under one declared contract.

The goal representation matters. Retirement of an old staging object followed by activation of a replacement is a two-obligation task. A compensation response does not satisfy the activation obligation. A final receipt for an earlier write likewise does not establish the current goal after an approved compensation. Scope and intent bindings connect evidence to the business condition being certified.

Provider capability differences change the appropriate action, not the intended goal. A valid atomic key can support a repeated attempt; an expired key cannot. A current negative query may help after a trusted settlement bound but can be insufficient while an older request is pending. Under commit-time authorization, a provider must enforce authority at that boundary. A local journal cannot retroactively fence an opaque remote operation.

\subsection{Evidence scope and validity}
The study measures recovery choices in an explicit, finite, known provider model, including candidate future attempt sequences and repeated final outcomes. This makes invisible outcome differences auditable and supports complete enumeration, but gives substantially more predictive information than a normal capability declaration. Default cross-operation product weights impose prior independence; the separate supplied-correlated-support diagnostic checks implementation handling, not estimation of real correlations. Real services may change contracts, introduce unmodeled failures, share mutable targets, or use batch semantics beyond the item-level profile. Observable model contradiction fails closed, while an unobservable omitted alternative can still invalidate retention and safety assumptions. Those conditions require independently qualified model coverage, not a repaired posterior chosen for convenience.

The implementation bounds operation counts and support size. Its explicit finite-tail proof covers scheduled effect prefixes, while the unknown-sentinel mechanism conservatively contains unresolved work. A null commit time records missing settlement information and remains distinct from every concrete future arrival time; persistence of that marker leaves completion uncertified. Additional invariants for harmless residual work or a stronger concretization abstraction define a separate qualification target beyond this profile.

The common safety check is an evaluation control. Zero newly introduced violations in the recorded finite set establishes coverage of those checks for those executions, not zero operational risk. The three assignments with prior or externally scheduled hazards remain absolute task failures. Separating attribution helps explain a recovery decision without erasing the unsafe history.

The four-unit synthetic fee saving and additional computation form a measurable resource trade-off. Complete per-task cost vectors retain zero and negative savings, and the compute-inclusive Pareto comparison preserves all reported axes. The symbolic break-even makes the required price ratio explicit. Translating these resource quantities into deployment economics requires isolated-host timing, provider prices, end-to-end latency and integration-cost measurements. The exhaustive reference validates the implemented finite-horizon objective within the declared provider model; the local-model study separately evaluates how its advice changes recovery behavior.

\subsection{Deployment and model evaluation}
The native persistence qualification establishes a real call path through existing public modules for request binding, durable replay, receipt reconstruction and ambiguous-delivery containment. The decision policies are evaluated in a separate finite runner. These tracks characterize persistence behavior and recovery selection, respectively. Evaluating their integrated deployment through independent agent runtimes and third-party providers would extend that evidence with matched executor and provider outcomes.

The local-model comparison holds model configuration, action schema, observations and per-episode limits constant, while adding planner advice in one condition. All proposals, rejections and observable tokens are retained. Its 29 held-out units are latent assignments within task families used during development, and both models share LM Studio. Every executed action passes the common shield, so the completion comparison evaluates advice within that shared safety layer. The confirmation establishes an advice-interface benefit in this frame; unseen task families, independent runtimes and real provider sandboxes define separate transfer tests. ReAct and Reflexion provide established agent-control context~\cite{yao2023react,shinn2023reflexion}, while the reported scripted controls are the explicitly defined finite-runner policies.

An external service experiment should use synthetic sandbox entities, documented provider capabilities, and response-loss injection without sending real communications or creating real charges. Original capability and author-imposed restrictions must be separated. Evaluation of task completion, delayed effects, and unknown outcomes should be independently bound to the resulting public evidence. Native demos remain qualification evidence until that matched interface and outcome coverage are supplied.

\subsection{Reproduction and material availability}
The retained materials include source-bound protocols, all controlled trajectories, a separately implemented adjudicator, generated tables, per-task costs, complete-root objective checks, budget-layer diagnostics, native qualification reports, local-model qualification records and the complete assignment-held-out confirmation. The recovery source/data candidate contains the synthetic controlled study and is retained locally for release review. No public identifier has yet been assigned to that candidate. The existing executor repository is separate from the new recovery study artifact.

\section{Conclusion}
Recovery from ambiguous tool effects is a task-level decision over evidence, remaining obligations, and provider capabilities. The known-model reference separates realized private alternatives from publicly supplied future attempt sequences, checks actions against the full compatible support, and requires prefix-stable finite-tail evidence rather than a plausible success message. Unknown-delay sentinels remain a separate conservative containment abstraction. The model includes intent-bound keys, authorization timing, dependencies, approved compensation, and resource limits.

The recorded finite study compares five policies on 37 assignments from 20 synthetic templates, with independent effect-history adjudication. The planner and the three active strong controls each complete 28 assignments safely and certifiably. Its resource trade-off consists of four saved synthetic fee units, two saved actions and increased selection computation, with monetary deployment costs requiring separate measurement. Persistent task costs and a full-vector/budget-layer audit make that trade-off and the implemented finite-horizon objective auditable. The retained first planner version demonstrates the importance of comparing completed search layers under a common evaluation objective.

The preregistered local-model comparison supplies end-to-end recovery evidence within the same known model. Planning assistance certifies \RecPlanner/\RecPairs{} paired episodes versus \RecEvidence/\RecPairs{} with evidence prompting, a \RecDifference{}-percentage-point difference ($p=\RecP$ under the registered family test), while using fewer proposals and observed model tokens. All \RecEpisodes{} trajectories pass independent replay. This result retains the common shield, authored scenario frame and one-runtime scope.

The operational conclusion is capability-specific: continue approved work when the public evidence and contract support it, and retain explicit unknown or handoff outcomes when they do not. The completed model study supports the evaluated advice interface. Generalization to new families, independent runtimes and third-party provider effects requires corresponding evidence.

\bibliographystyle{ACM-Reference-Format}
\bibliography{references}
\clearpage
\appendix
\section{Artifact and Protocol Detail}
\begin{table}[htbp]
\caption{Assignment coverage and certified safe completions by public template. Template names are identifiers, not independent population samples or proof that every named action was executed.}
\label{tab:coverage}
\small
\begin{tabularx}{\linewidth}{@{}Yrrrr@{}}\toprule
Public template & Assignments & Rules & Greedy & Joint \\ \midrule
billing--chain & 2 & 2 & 2 & 2 \\
stock--diamond & 4 & 4 & 4 & 4 \\
file--fork & 1 & 1 & 1 & 1 \\
approval--independent & 2 & 2 & 2 & 2 \\
batch--items & 2 & 2 & 2 & 2 \\
valid--key--recovery & 2 & 2 & 2 & 2 \\
expired--key--recovery & 2 & 2 & 2 & 2 \\
delayed--query & 2 & 2 & 2 & 2 \\
inflight--query & 2 & 2 & 2 & 2 \\
no--query--unknown & 3 & 0 & 0 & 0 \\
late--commit--after--deadline & 2 & 0 & 0 & 0 \\
ambiguous--compensation & 2 & 2 & 2 & 2 \\
failed--continuation & 1 & 1 & 1 & 1 \\
configuration--retirement & 2 & 2 & 2 & 2 \\
retirement--already--inflight & 2 & 2 & 2 & 2 \\
existing--duplicate--risk & 2 & 0 & 0 & 0 \\
authorization--at--acceptance & 1 & 1 & 1 & 1 \\
authorization--at--commit--enforced & 1 & 1 & 1 & 1 \\
authorization--at--commit--unenforced & 1 & 0 & 0 & 0 \\
budget--limited--independent & 1 & 0 & 0 & 0 \\
\bottomrule
\end{tabularx}
\end{table}

\begin{table}[htbp]
\caption{Persistent per-template comparison on the original 185-row mechanism-003 collection. Positive savings favor joint planning versus capability rules; extra selection seconds are host wall time, not CPU billing. Break-even is seconds per saved synthetic fee unit, not a measured monetary benefit. Full per-assignment costs and Pareto sets are in the separate diagnostic artifact.}
\label{tab:task-cost}
\small
\begin{tabularx}{\linewidth}{@{}Yrrrr@{}}\toprule
Template & Fee saved & Actions saved & Extra seconds & Break-even \\ \midrule
billing--chain & 0 & 0 & 1.494 & -- \\
stock--diamond & 0 & 0 & 4.302 & -- \\
file--fork & 0 & 0 & 0.475 & -- \\
approval--independent & 0 & 0 & 0.539 & -- \\
batch--items & 0 & 0 & 0.714 & -- \\
valid--key--recovery & 1 & -1 & 0.065 & 0.065 \\
expired--key--recovery & 0 & 0 & 0.058 & -- \\
delayed--query & 0 & 0 & 0.176 & -- \\
inflight--query & 0 & 0 & 0.212 & -- \\
no--query--unknown & 0 & 0 & 0.070 & -- \\
late--commit--after--deadline & 0 & 0 & 0.024 & -- \\
ambiguous--compensation & 0 & 0 & 0.000 & -- \\
failed--continuation & 0 & 0 & 0.132 & -- \\
configuration--retirement & 1 & 1 & 0.557 & 0.557 \\
retirement--already--inflight & 1 & 1 & 0.554 & 0.554 \\
existing--duplicate--risk & 0 & 0 & 0.131 & -- \\
authorization--at--acceptance & 0 & 0 & 0.002 & -- \\
authorization--at--commit--enforced & 0 & 0 & 0.013 & -- \\
authorization--at--commit--unenforced & 1 & 1 & 0.060 & 0.060 \\
budget--limited--independent & 0 & 0 & 0.006 & -- \\
\bottomrule
\end{tabularx}
\end{table}

\subsection{Public task schema}
An operation binds an ASCII identifier, target, request digest, dependencies, effect class, approved goal count, fees, durations, and the provider capability fields in Table~\ref{tab:contracts}. Goal counts are zero or one. The finite support records positive weights, initial copies, historical commit and visibility times, key records, existing pending effects, and declared write/compensation outcome sequences. A compensated initial history explicitly records both its historical write and its removal. A valid historical key can therefore remain while the current count is zero.

Outcome sequences specify successful or failed attempts, response type, finite nonnegative integer delay, and effect multiplicity. Their last entry repeats for every later attempt; they are complete known-model sequences, not capability-only hints. The initializer uses a cross-operation product prior; a supplied public joint support is required for correlations. Subsequent attempts use finite declared outcome delays. An initially pending \texttt{null} commit time is an unknown sentinel that does not advance or certify completion, not concrete simulation of arbitrary unbounded delay. Initial finite effects and pending events are replayed by event time, dependency order for equal times, and stable event identifier. A topological initial replay alone would conceal an earlier child effect that violated its predecessor condition, so temporal order is preserved.

\subsection{Evidence closure}
The controlled protocol binds the provider environment, scenario construction, recovery policies, Python adjudicator, collector, and checker. The row chain includes sequence, previous-row hash, and the full episode. The final manifest binds the exact protocol and raw stream. Verification checks source identity, frozen configuration, denominator, and order before replaying policies and independently adjudicating their outputs.

These hashes support consistency and provenance under a fixed source set. They are not an append-only remote storage service or a defense against an authorized person rewriting both source and evidence. Native qualification has its own source bindings and raw logs. Upstream gitlinks are recorded by their pinned commit and marked uninitialized when not needed by the executed demo; no unexecuted submodule is represented as tested code.

\subsection{Reproduction commands}
With the recorded Node.js and Python versions available, the scientific collection can be recreated in a fresh numbered evidence directory:
\begin{verbatim}
node --test artifact/tests/*.test.mjs
PYTHONDONTWRITEBYTECODE=1 python3 -m unittest discover \
  -s artifact/tests -p 'test_*.py'
node artifact/run-study.mjs results/mechanism-005
node artifact/check-study.mjs results/mechanism-005
node artifact/render-study.mjs results/mechanism-005
node artifact/audit-development.mjs \
  results/development-diagnostics-002 results/mechanism-005
node artifact/audit-development.mjs --render \
  results/development-diagnostics-002
latexmk -pdf -interaction=nonstopmode -halt-on-error main.tex
\end{verbatim}
The recorded source snapshot, not a future modified checkout, defines the corresponding collection. Existing numbered directories are never overwritten. Python independently replays effects and observations before a table renderer accepts results. Native executor tests require the existing package build; upstream qualifications require separately obtained pinned repositories. The finite controlled labels are reconstructed from the separate Python replay, with native qualification results retained in their own track. The numbered diagnostic root freezes its parameterizations and source bytes before evaluation; its verifier re-enumerates all objective/budget checks and recomputes task costs from bound original records.

\subsection{Evaluation distinctions}
Actual goal satisfaction, a public terminal certificate, and absolute safe completion answer different questions. The actual goal uses the private replayed state; the certificate uses all states consistent with public observations; safe completion additionally requires the full history and deadline constraints. A handoff is counted as handoff, not task completion. Compensation is counted as a public action and never silently changes the approved goal.

The records retain logical durations, collector wall times, and one measured selection time per decision separately. These timings include implementation and host-load effects and do not substitute for provider latency or CPU billing. Replay compares semantic outcomes while excluding nondeterministic timing fields. Total fees, calls, waits, compensation, full-depth decisions, planner/fallback selections, and truncations are derived from complete rows. The expanded small-domain audit independently enumerates all root-action vectors, compares each complete layer, allows tied optima, and separately checks partial-layer exclusion and rule fallback. Complete task-cost rows retain all five policies; Pareto axes maximize certified completion and minimize introduced violations, fee, actions, logical time, and selection wall time. This is finite-horizon implementation validation and symbolic cost sensitivity, not full-task or production optimality.

\subsection{Model protocols and execution deviations}
The archived qualification bundle and the archived confirmation bundle bind the two numbered outcome runs and their recorded summary. The preregistration specifies the exhaustive pilot-assignment complement and primary family test before collection. During integration, both outcome runs are verified without writing into them, and paired counts and the exact test are recomputed from the exported episode/pair records.

The first confirmatory preflight failed when runtime JIT unloading invalidated a model key. Its record is retained. The scientific freeze was unchanged for the successful launch. The handoff source manifest names \nolinkurl{start-confirmatory-001.sh}, while the actual launcher is \nolinkurl{start-confirmatory-002.sh}; its SHA-256 is \nolinkurl{8ef0478be31c262421c5b89577f6659c1471cdd34295149cf9cc71688bbf1c45}. The release records this launch discrepancy separately from the frozen scientific source and does not claim an identical launcher manifest.

\end{document}